%% file: acl_latex.tex
\pdfoutput=1

\documentclass[11pt]{article}

\usepackage[preprint]{acl}

\usepackage{times}
\usepackage{latexsym}

\usepackage[T1]{fontenc}

\usepackage[utf8]{inputenc}

\usepackage{inconsolata}

\input{sections/headers}

\title{On Weak-to-Strong Generalization and $f$-Divergence}

\author{
\textbf{Wei Yao\textsuperscript{1{$\star$}}},
\textbf{Gengze Xu\textsuperscript{1{$\star$}}},
\textbf{Huayi Tang\textsuperscript{1}},
\textbf{Wenkai Yang\textsuperscript{1}}, \\
\textbf{Donglin Di\textsuperscript{2}},
\textbf{Ziqiao Wang\textsuperscript{3}},
\textbf{Yong Liu\textsuperscript{1}$^{\dag}$}
\\
\\
 \textsuperscript{1}Gaoling School of Artificial Intelligence, Renmin University of China,
 \\
 \textsuperscript{2}Li Auto Inc., 
 \\
 \textsuperscript{3}School of Computer Science and Technology, Tongji University.
\\
\tt\footnotesize\{wei.yao,xugengze,liuyonggsai\}@ruc.edu.cn\\
}

\begin{document}

\maketitle

\let\thefootnote\relax\footnotetext{$^\star$ Equal contribution\hspace{3pt} \hspace{5pt}$^{\dag}$ Corresponding author\hspace{5pt}}

\begin{abstract}
Weak-to-strong generalization (W2SG) has emerged as a promising paradigm for stimulating the capabilities of strong pre-trained models by leveraging supervision from weaker supervisors. 
To improve the performance of the strong model, existing methods often require additional weak models or complex procedures, leading to substantial computational and memory overhead.
Motivated by the effectiveness of $f$-divergence loss in various machine learning domains, we introduce $f$-divergence as an information-theoretic loss function framework in W2SG.
Our theoretical analysis reveals fundamental limitations and equivalence of different $f$-divergence losses in W2SG, supported by sample complexity bounds and information-theoretic insights. 
We empirically demonstrate that $f$-divergence loss, which generalizes widely-used metrics like KL divergence, effectively improves generalization and noise tolerance of the strong model in practice.
\end{abstract}

\input{sections/introduction}

\input{sections/related_work}

\input{sections/preliminaries}

\input{sections/theory}

\input{sections/experiments}

\input{sections/conclusion}


\bibliography{custom}


\newpage
\onecolumn
\appendix

{\LARGE \centering \textbf{Contents} \par}
\startcontents[section]
\printcontents[section]{l}{1}{\setcounter{tocdepth}{2}}


{\LARGE \centering \textbf{Appendix} \par}

\input{sections/appendix}

\end{document}

%% file: sections/headers.tex
\usepackage[utf8]{inputenc} 
\usepackage[T1]{fontenc}    
\usepackage{hyperref}       
\usepackage{url}            
\usepackage{booktabs}       
\usepackage{amsfonts}       
\usepackage{nicefrac}       
\usepackage{microtype}      
\usepackage{xcolor}         
\usepackage{array}   
\usepackage{titletoc}

\usepackage{amsmath}
\usepackage{amssymb}
\usepackage{mathtools}
\usepackage{amsthm}
\usepackage{bm}

\usepackage{multirow}
\usepackage{graphicx}

\usepackage{subcaption}  

\newtheorem{theorem}{Theorem}[section]
\newtheorem{corollary}{Corollary}[section]

\newtheorem{definition}{Definition}[section]

\newtheorem{lemma}{Lemma}[section]
\newtheorem*{remark}{Remark}
\newtheorem*{proofsketch}{Proof sketch}

\newcommand{\fdiv}{\mathrm{D}_{f}}

\newcommand{\tv}{\mathrm{D}_{\mathrm{TV}}}

\newcommand{\R}{\mathbb{R}}
\newcommand{\E}{\mathbb{E}}
\newcommand{\cX}{\mathcal{X}}
\newcommand{\cY}{\mathcal{Y}}

\newcommand{\cC}{\mathcal{C}}

\newcommand{\expect}{\mathbb{E}}

\newcommand{\cO}{\mathcal{O}}

\newcommand{\cF}{\mathcal{F}}
\newcommand{\cP}{\mathcal{P}}
\newcommand{\bE}{\mathbb{E}}

\newcommand{\argmin}{\text{argmin}}
\newcommand{\eps}{\varepsilon}
\newcommand{\dist}{R_f}
\newcommand{\disthat}{\hat{R}_f}

\usepackage{url}

\usepackage[capitalize,noabbrev]{cleveref}

\allowdisplaybreaks[4]

%% file: sections/introduction.tex
\section{Introduction}

AI alignment focuses on ensuring that AI systems behave in accordance with human capabilities and intentions~\citep{ouyang2022training}. However, as AI models become increasingly powerful, potentially exceeding human-level capabilities, standard alignment methods may become insufficient. This challenge motivates the concept of \textit{superalignment}~\citep{openai_superalignment}: the problem of aligning superhuman AI systems using human oversight or weaker AI models~\citep{kim2024road,kim2025research}.

For an initial exploration of this problem, recent work has demonstrated that strong models can outperform their weak supervisors when fine-tuned using weak supervision signals, a paradigm known as weak-to-strong generalization~\citep{burns2023weak}.
To further enhance the performance of strong models, existing approaches primarily focus on two strategies: (1) iterative refinement~\citep{lyu2024macpo,ye2025iterative,lang2025debate} and (2) the integration of additional weak supervisors~\citep{agrawal2024ensemw2s,sang2024improving,liu2024co,cui2024bayesian}. 
Although these methods show promise, they share a critical drawback: reliance on either multiple weak models or intricate training pipelines, leading to substantial computational and memory overhead.
Therefore, we raise an intriguing and challenging question:
\textit{Can we improve the performance of strong model in W2SG without requiring additional costly modules?}

To address this question, this paper presents a theoretically-grounded framework leveraging $f$-divergence~\citep{renyi1961measures}, a fundamental class of statistical measures for quantifying discrepancy of probability distributions~\citep{sason2016f}. 
It includes numerous prominent divergence metrics, including Kullback-Leibler (KL) divergence, reverse KL divergence, Jensen-Shannon (JS) divergence, Jeffreys divergence, squared Hellinger distance, Pearson $\chi^2$ divergence, and Total Variation (TV) distance, offering exceptional versatility. 
Notably, integrating $f$-divergence into the loss function requires neither additional models nor complex training procedures, enabling efficient performance enhancement for strong models in W2SG settings. 
Furthermore, the effectiveness of this approach is further evidenced by its successful application across diverse machine learning domains, including classification~\citep{novello2024fdivergence}, generative modeling~\citep{nowozin2016f}, and imitation learning~\citep{ke2021imitation}, where it has consistently demonstrated empirical superiority.
Given these practical successes, investigating $f$-divergence's potential to advance the W2SG paradigm emerges as a natural and promising research direction.

Theoretically, our paper extends theoretical analysis on cross-entropy loss~\citep{yao2025revisiting} in W2SG and reveals fundamental limitations of $f$-divergence losses.
Specifically, we prove that if we use $f$-divergence loss, the strong model's performance is intrinsically limited by the weak supervisor and the disagreement between them.
This result is in line with~\citet{yao2025revisiting}.
Moreover, to systematically capture the stochastic nature of model optimization and data sampling, we develop an information-theoretic analysis demonstrating that the strong model's performance asymptotically approaches the weak model's performance unless constrained optimization methods are applied.
Interestingly, we also theoretically discover the equivalence between different $f$-divergence losses when combined with the confidence-enhancing regularization introduced by~\citet{burns2023weak}.

Empirically, due to the inherent noise in W2SG arising from weak supervision, we evaluate both standard performance and the robustness of $f$-divergence losses to label noise, conducting a systematic comparison with the conventional cross-entropy (CE) loss. 
To assess this in our experiments, we gradually increase the noise level, i.e., the proportion of label-noised weak supervision. 
Our analysis reveals several key observations. 
\begin{enumerate}
    \item Jeffreys divergence and reverse KL divergence, both of which inherently incorporate the reverse KL term, consistently outperform CE loss. This extends the earlier findings of~\citet{yao2025revisiting} regarding the superiority of reverse KL divergence in W2SG.
    \item In moderate label noise conditions (a noise level from 10\% to 40\%), Hellinger distance demonstrates the most robust performance among all considered divergences.
    \item In the extreme noise setting (a noise level of 50\%), Pearson $\chi^2$ divergence and JS divergence show consistent improvements over standard CE loss.
\end{enumerate}
We further investigate the impact of $f$-divergence when incorporating confidence-enhancing regularization~\citep{burns2023weak}, observing that $f$-divergence losses yield promising results.
In general, all of these empirical results strongly suggest that $f$-divergence functions represent promising alternatives to conventional loss functions in W2SG, particularly in scenarios involving noisy labels.

%% file: sections/related_work.tex
\section{Related Work}

\paragraph{Weak-to-strong generalization.}

To address the challenge of superalignment~\citep{openai_superalignment}, recent work~\citep{burns2023weak} introduced a novel paradigm termed weak-to-strong generalization (W2SG). 
Their key finding demonstrates that when strong pre-trained language models are fine-tuned using supervision signals from weaker models, they consistently surpass the performance of their weak supervisors. 
This intriguing phenomenon suggests that strong students can effectively extract and amplify the useful knowledge embedded in weaker teachers.
Building upon this discovery, subsequent research has empirically~\citep{pawelczyk2024generalizing,yang-etal-2024-weak,guo2024vision,shin2024weak,zhou2025weak} or theoretically~\citep{lang2024theoretical,somerstep2024statistical,somerstep2025transfer,wu2024provable,charikar2024quantifying,yao2025revisiting,yao2025understanding,mulgund2025relating,xue2025representations,medvedev2025weak,xu2025on} explored the properties of W2SG.
Additionally, various techniques are developed to enhance the strong model's performance.
Popular approaches include iterative updating~\citep{lyu2024macpo,ye2025iterative,lang2025debate}, and incorporating more weak supervisors~\citep{agrawal2024ensemw2s,sang2024improving,liu2024co,cui2024bayesian}. 
Although these methods show promising results, they often require additional weak models or involve complex computational procedures, leading to significant time and space overhead.
In this work, we propose a novel perspective grounded in information theory. 
Specifically, we introduce $f$-divergence as a theoretically-principled loss function framework for W2SG.
Our comprehensive empirical evaluation demonstrating the benefits of this loss family, particularly in terms of generalization capability and noise toleration.

\paragraph{$f$-Divergence.}
$f$-divergence~\citep{renyi1961measures} is a fundamental class of statistical measures that quantifies the difference between two probability distributions~\citep{sason2016f,polyanskiy2025information}. 
Its applications span a wide range of tasks, such as classification~\citep{wei2021when,zhong2023learning,novello2024fdivergence}, generative model training~\citep{nowozin2016f,yu2020training}, imitation learning~\citep{ghasemipour2020divergence,ke2021imitation} and variational inference~\citep{wan2020f}.
Theoretical studies have further explored its applications, including $f$-divergence minimization~\citep{li2024convergence}, empirical risk minimization with $f$-divergence regularization~\citep{daunas2024equivalence,daunas2025generalization} and $f$-divergence principled domain adaptation~\citep{wang2024on}.
Recently, $f$-divergence loss and regularization techniques have been increasingly employed in key NLP applications, particularly in knowledge distillation~\citep{wen-etal-2023-f,gu2024minillm} and preference optimization~\citep{pmlr-v202-go23a,wang2023beyond,xu2025fpo}.
However, their applicability to W2SG remains underexplored. 
In this paper, we investigate the potential of $f$-divergence in W2SG, demonstrating its effectiveness in improving accuracy and noise toleration.

%% file: sections/preliminaries.tex
\section{Preliminaries}

\subsection{Classification via \texorpdfstring{$f$}{f}-Divergence Losses}

We first introduce the definition of $f$-divergence.

\begin{definition}[$f$-divergence~\citep{liese2006divergences}] \label{def:fdiv}
Let $P$ and $Q$ be two distributions on $\Theta$. Let $f: \mathbb{R}_{+} \rightarrow \mathbb{R}$ be a convex function with $f(1)=0$. 
If $P \ll Q$ \footnote{$P \ll Q$ means that $P$ is absolutely continuous with respect to $Q$, i.e., if $Q(x)=0 \Longrightarrow P(x)=0$ for all measurable sets $x \subseteq \Theta$.}
, then $f$-divergence is defined as $\fdiv(P \| Q) \triangleq \mathbb{E}_Q\left[f\left(\frac{d P}{d Q}\right)\right]$.
\end{definition}


The $f$-divergence family covers a broad class of commonly used divergences (shown in~\cref{table:fdiv}), including KL divergence, reverse KL divergence, JS divergence, TV distance, etc.

\begin{table*}[ht]
\vspace{-5pt}
\centering
\caption{$f$-divergences and their corresponding functions $f(x)$. The distribution $M=\frac{P+Q}{2}$.}
\label{table:fdiv}
\begin{tabular}{c|c|c}
\hline
\textbf{$f$-Divergences}       & $f(x)$                                                                 & \textbf{$D_f(P\|Q)$} \\ \hline
KL divergence                  & $x \ln x$                                                              & $\int P(x)\ln\frac{P(x)}{Q(x)}dx$ \\
Reverse KL divergence          & $- \ln x$                                                              & $\int Q(x)\ln\frac{Q(x)}{P(x)}dx$ \\
JS divergence & $\frac{1}{2}\left(x \ln x-(x+1) \ln \left(\frac{x+1}{2}\right)\right)$ & $\frac{1}{2}D_{KL}(P \| M) + \frac{1}{2}D_{KL}(Q \| M)$ \\
Jeffreys divergence            & $(x-1) \ln x$                                                          & $\int (P(x)-Q(x))\ln\frac{P(x)}{Q(x)}dx$ \\
Squared Hellinger distance     & $1-\sqrt{x}$                                                           & $\int (\sqrt{P(x)}-\sqrt{Q(x)})^2dx$ \\
Pearson $\chi^2$ divergence    & $(x-1)^2$                                                              & $\int \frac{(P(x)-Q(x))^2}{Q(x)}dx$ \\
TV distance  & $\frac{1}{2}\left| x-1 \right|$                                        & $\frac{1}{2}\int |P(x)-Q(x)|dx$ \\ \hline
\end{tabular}
\vspace{-10pt}
\end{table*}

We consider $k$-classification tasks.
Given the data domain $\cX \subseteq \R^d$ and output domain $\cY \subseteq \R^k$, let the model space be $\cF: \cX \to \cY$. 
The model is equipped with a softmax module, which ensures that its outputs form a valid probability distribution, i.e., $\forall y = (y_1, \cdots, y_k)^T \in \cY$, there holds $\sum_{i=1}^k y_i=1$ and $0 < y_i < 1$.
We define the disagreement between two models using $f$-divergence:

\begin{definition}[Disagreement using $f$-divergence]
Given the data distribution $\cP$ and two models $g,h \in \cF$. 
Define the disagreement $\dist,\disthat: \cF \times \cF \to \R_0^+$ between $g$ and $h$ as:
\begin{align}
& \dist(g,h) \triangleq \bE_{x \sim \cP} \left[ \fdiv(g(x) \| h(x)) \right], \\
& \disthat(g,h) \triangleq \frac{1}{n} \sum_{j=1}^n \fdiv(g(x_j) \| h(x_j)),
\end{align}
where $\{ x_j \}_{j=1}^n$ are $n$ i.i.d. samples.
\end{definition}

If $g$ is the labeling function, $h$ is a model to be optimized, and $\fdiv$ is the KL divergence, then minimizing $\dist(g,h)$ can be interpreted as training model $h$ using KL divergence loss. This optimization is equivalent to minimizing the cross-entropy loss, as the two loss functions differ only by a constant.
Furthermore, minimizing $\disthat(g,h)$ corresponds to minimizing the average error over the training data—a learning paradigm called Empirical Risk Minimization (ERM)~\citep{vapnik1999overview}.

\subsection{Weak-to-Strong Generalization}

In the context of $k$-classification tasks in W2SG, we focus on the fine-tuning phase after pre-training.
The labeling function $G^\star$ maps data $x$ to its label $G^\star(x)$.
The strong model aims to learn a predictor $G^{sw} = g \circ h^s$, where $h^s$ is a fixed strong model representation and $g \in \cF_{s}$ is a task-specific function from a hypothesis class $\cF_{s}$.
Following the conventional setting of AI alignment~\citep{ouyang2022training}, the model is fine-tuned using ground truth data:
\begin{align}
    \label{eqn:alignment-population-minimizer}
    g^{s} = \argmin_{g \in \cF_{s}}\; \dist(g \circ h^s, G^\star).
\end{align}
However, in the super-alignment scenario~\citep{openai_superalignment}, weak supervision is ultimately provided by human input. To investigate this paradigm, the W2SG framework~\citep{burns2023weak} employs predictions from a weak model to guide the training of its stronger student:
\begin{align}
    \label{eqn:fsw-population-minimizer}
    g^{sw} = \argmin_{g \in \cF_{s}}\; \dist(g \circ h^s, G^w),
\end{align}
where $G^w$ is a given weak model.
The fundamental goal of this framework is obtaining $G^{sw}$ with a small $\dist(G^{sw}, G^\star)$.
In practice, we use $n$ i.i.d. samples and conduct ERM for W2SG:
\begin{align} \label{eqn:erm}
    \hat{g}^{sw} = \argmin_{g \in \cF_s} \disthat(g \circ h^s, G^w).
\end{align}

Intuitively, as the amount of training data grows (i.e., $n \to +\infty$), the empirical misfit $\disthat(g \circ h^s, G^w)$ serves as an accurate approximation of the population-level misfit $\dist(g \circ h^s, G^w)$~\citep{charikar2024quantifying,yao2025revisiting}.

%% file: sections/theory.tex
\section{Theoretical Analysis}

\subsection{Fundamental Limitations of $f$-Divergence Losses}

Following~\citet{yao2025revisiting,yao2025understanding}, we first analyze the fundamental limitations of $f$-divergence losses.

\begin{theorem}[Proved in~\cref{proof:upper_lower_fdiv}] \label{theorem:upper_lower_fdiv}
Given the data domain $\cX$, output domain $\cY$, models $G^w, G^\star$ and disagreement $\dist(\cdot, \cdot)$ defined above. 
For any strong model $G^{sw}$, there holds
\resizebox{\linewidth}{!}{
\begin{minipage}{1.05\linewidth}
\begin{multline} \label{eq:thm}
    \left| \dist(G^{sw}, G^\star) - \dist(G^w, G^\star) \right| \\ = \cO \left( \sqrt{\dist(G^{sw}, G^w)} \right).
\end{multline}
\end{minipage}
}
\end{theorem}

\begin{proofsketch}
We employ information-theoretic results on the relationship between $f$-divergence and TV distance~\citep{gilardoni2010pinsker}, combined with fundamental mathematical tools such as the mean value theorem, Jensen’s inequality, and the properties of $f$-divergences, to derive our result.
\end{proofsketch}

The theoretical results align with prior theoretical analysis~\citep{yao2025revisiting,yao2025understanding} of KL divergence and cross-entropy loss: the strong model's behavior is fundamentally constrained by both the W2SG minimization objective and the weak supervisor's quality. 
Excessive optimization in W2SG drives $\dist(G^{sw}, G^w) \to 0$, causing the strong model to overfit to weak supervision, whereas controlled optimization (e.g., early stopping~\citep{burns2023weak,yao2025understanding} to maintain moderate $\dist(G^{sw}, G^w)$) combined with a higher-quality weak model (reducing $\dist(G^w, G^\star)$) can maximize the strong student's potential. 
Nevertheless, the strong model's generalization error remains fundamentally bounded by $\dist(G^{sw}, G^w)-\cO \left( \sqrt{\dist(G^{sw}, G^w)} \right)$, reflecting an intrinsic limitation of both $f$-divergence losses and the cross-entropy loss commonly employed in W2SG frameworks~\citep{burns2023weak,yao2025revisiting}.
This barrier persists in current W2SG frameworks unless additional assumptions are introduced, such as convexity~\citep{charikar2024quantifying,mulgund2025relating,yao2025revisiting} and Gaussian data distributions~\citep{ildiz2025highdimensional,wu2024provable,somerstep2025transfer}.

To broaden the scope of our theory, we extend the above results to the ERM framework in~\cref{eqn:erm}, where the strong model is optimized using $n$ i.i.d. samples labeled by a weak supervisor.

\begin{corollary}[Proved in~\cref{proof:upper_lower_fdiv_sample}] \label{theorem:upper_lower_fdiv_sample}
Given the data domain $\cX$, output domain $\cY$, models $G^w,G^{sw},G^\star$ and disagreement $\dist(\cdot, \cdot)$ defined above.
For any strong model $\hat{G}^{sw}=\hat{g}^{sw} \circ h^s$, with probability at least $1-\delta$ over $n$ i.i.d. samples, there holds
\resizebox{\linewidth}{!}{
\begin{minipage}{1.05\linewidth}
\begin{multline} \label{eq:coro}
    \left| \dist(\hat{G}^{sw}, G^\star) - \dist(G^w, G^\star) \right| = \\ \cO \left( \sqrt{\disthat(\hat{G}^{sw}, G^w)} \right) + \cO \left(\frac{\cC_{\cF_s}}{n}\right)^{\frac{1}{4}} + \cO \left(\frac{\log \frac{1}{\delta}}{n}\right)^{\frac{1}{4}},
\end{multline}
\end{minipage}
}
where $\cC_{\cF_s}$ is a constant capturing the complexity of the function class $\cF_s$, and the asymptotic notation is with respect to $n \to \infty$.
\end{corollary}


The first term on the right-hand side of~\cref{eq:coro} is determined by the training error $\disthat(\hat{G}^{sw}, G^w)$ in W2SG, which we expect to be small in practice.
The remaining two terms represent the key differences between our result and~\cref{eq:thm}. 
$\frac{\cC_{\cF_s}}{n}$ depends on the complexity of the fine-tuning function class $\cF_s$, while both this term and the confidence term $\frac{\log \frac{1}{\delta}}{n}$ vanish as the number of training samples $n \to \infty$.

Our theoretical analysis reveals an important limitation of W2SG: \textit{when allowing the strong model to overfit to weak supervision with unlimited training data, its performance converges to that of the weak model}, thereby failing to fully exploit the strong model's capacity.
This theory further offers theoretical justification for employing early stopping as a means to mitigate overfitting in initial considerations of W2SG~\citep{burns2023weak}.

To further systematically capture the stochastic nature of model optimization and data sampling in W2SG, we adopt an information-theoretic generalization analysis~\citep{xu2017information,bu2020tightening,wang2022information,tang2023information}.
This theoretical paradigm quantifies the impact of the stochastic factors by measuring the dependence between learned parameters and training data.

\begin{theorem}[Proved in~\cref{proof:inf_theoretic}] \label{corollary:information}
Given the data domain $\cX$, output domain $\cY$, models $G^w, G^\star$ and disagreement $\dist(\cdot, \cdot)$ defined above. 
Let $Z=(X_1, \cdots, X_n)$ and $\theta$ be the parameters of the strong model $\hat{G}_\theta^{sw}$.
For any $\hat{G}^{sw}_\theta$, there holds
\resizebox{\linewidth}{!}{
\begin{minipage}{1.05\linewidth}
\begin{multline} \label{eq:infor}
    \left| \mathbb{E}_{\theta,Z} \left[ \dist(\hat{G}^{sw}_\theta,G^\star) - \dist(G^w,G^\star) \right] \right| = \\ \cO \left(\sqrt{\mathbb{E}_{\theta,Z} \left[ \disthat(\hat{G}^{sw}_\theta, G^w) \right]}\right) + \cO \left(\frac{I(\theta;Z)}{n} \right)^{\frac{1}{4}}.
\end{multline}
\end{minipage}
}
\end{theorem}

\begin{proofsketch}
    We leverage information-theoretic bounds from~\citep{wang2022information} and properties of mutual information to establish our final results.
\end{proofsketch}

The expectation represents the theoretical average. In this context, $\mathbb{E}_{\theta,Z}$ denotes the expected value computed over both the model parameters $\theta$ and the training dataset $Z$. This expectation captures the average behavior across all possible realizations, characterizing the fundamental stochastic nature of both model optimization and data sampling.
In the W2SG framework, it is critical to prevent the strong model from overfitting to weak supervision~\citep{burns2023weak}. Specifically, the model parameters $\theta$ should not encode excessive information about the finite training dataset, meaning the mutual information $I(\theta;Z)$ must remain constrained. This aligns with information-theoretic generalization analysis~\citep{russo2016controlling,xu2017information} and information bottleneck principle~\citep{tishby2015deep,shwartz2017opening}, which similarly advocate for controlling mutual information to ensure better generalization.
Moreover, we can draw a conclusion analogous to~\cref{theorem:upper_lower_fdiv_sample}: under the conditions that (1) there are infinite training data ($n \to \infty$), and (2) the strong model exhibits perfect overfitting to the weak supervision (i.e., $\disthat(\hat{G}^{sw}_\theta, G^w)$ asymptotically goes to zero), then the strong model’s performance degenerates to that of the weak model.
Consequently, this restricts the achievable performance gains in W2SG. To mitigate this issue, it is essential to maintain a non-trivial disagreement between the strong and weak models while ensuring an appropriate optimization process that prevents overfitting, as highlighted in prior work~\citep{burns2023weak,yao2025revisiting}.


\input{figures/fig_1}

\subsection{Equivalence of $f$-Divergence Losses with Regularization}

As highlighted by~\citet{burns2023weak}, we explore an alternative approach called ``Auxiliary Confidence Loss'': an additional regularization term designed to enhance the strong model’s confidence in its predictions.
Let $\hat{G}_\theta^t$ be the hardened strong model predictions using a threshold $t$, i.e., for any $x$: $\hat{G}_\theta^t(x)= \mathbb{I}(G^{sw}_\theta(x) > t) \in \{ 0,1 \}$,
where $\mathbb{I}(\cdot)$ is the indicator function.
Then the regularization term is formulated as the cross-entropy loss between the predictions of $\hat{G}_\theta^t$ and $G_\theta^{sw}$ over the training set.
This regularization serves to mitigate overfitting to weak supervision, thereby improving the overall performance of the strong model~\citep{burns2023weak}.
In the subsequent theoretical analysis, we demonstrate that minimizing an $f$-divergence loss with this regularization is equivalent to minimizing another $f$-divergence loss with an appropriately transformed regularization term.

\begin{theorem}[Proved in~\cref{proof:equivalence}] \label{theorem:equivalence}
Let $f_1$ and $f_2$ be two strictly convex and differentiable functions satisfying the conditions of $f$-divergence in~\cref{def:fdiv}. Then there holds
\resizebox{\linewidth}{!}{
\begin{minipage}{1.05\linewidth}
\begin{multline*}
\min_{\theta} \hat{R}_{f_1} \left(G^{sw}_\theta, G^w\right) + \alpha \cdot \expect_\theta \left[ \frac{1}{n} \sum_{i=1}^n L(\theta, x_i) \right] = \\ \min_{\theta} \hat{R}_{f_2} \left(G^{sw}_\theta, G^w\right) + \alpha \cdot \expect_\theta \left[ \frac{1}{n} \sum_{i=1}^n v_i\left( L(\theta, x_i) \right) \right],
\end{multline*}
\end{minipage}
}
where $\alpha>0$ is the regularization strength, the regularization $L(\theta, x_i)$ is the cross-entropy loss between the prediction probability distribution of $\hat{G}_\theta^t(x_i)$ and $G^{sw}_\theta(x_i)$, and $v_i(\cdot)$ is a proper transformation function.
\end{theorem}

\begin{proofsketch}
    Our proof is motivated by recent advances in information theory~\citep{daunas2024equivalence}.
    We employ the properties of $f$-divergence, Gateaux differential, and Lagrange multiplier to analyze the Radon-Nikodym derivative and establish the result.
\end{proofsketch}

The $f$-divergence loss is equivalent to a different $f$-divergence loss with an appropriate transformation of the confidence-enhancing regularizer.
This result indicates the consistency of the target of $f$-divergence loss after the regularization is added.
However, this proof relies on the assumption of optimal transformation functions. In practice, hyperparameter tuning for $\alpha$ may be non-trivial across different divergence choices, potentially leading to deviations from the theoretical predictions.

%% file: figures/fig_1.tex
\begin{figure*}[t]
    \centering
    \begin{subfigure}[b]{\textwidth}  
        \includegraphics[width=\linewidth]{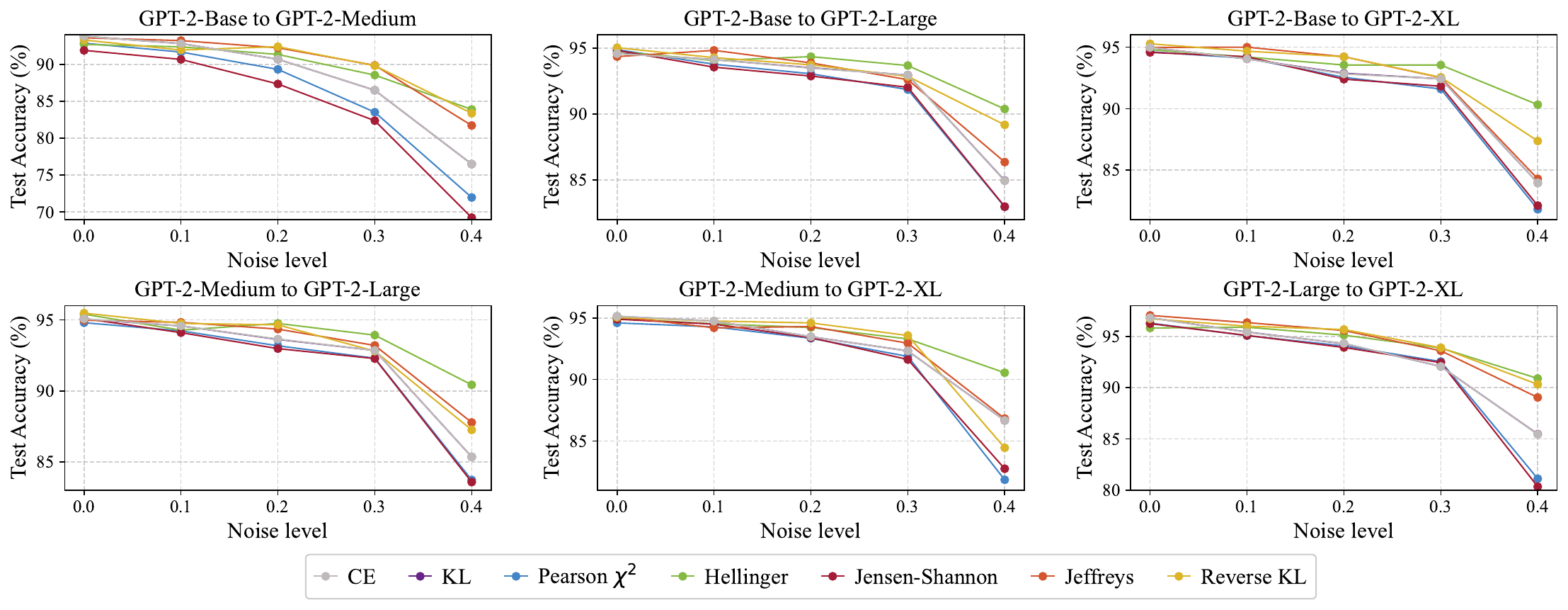}
        \caption{Results of GPT-2-series on CAI-Harmless}
        \label{fig:cai-noise}
    \end{subfigure}
    \bigskip  
    \begin{subfigure}[b]{\textwidth}  
        \includegraphics[width=\linewidth]{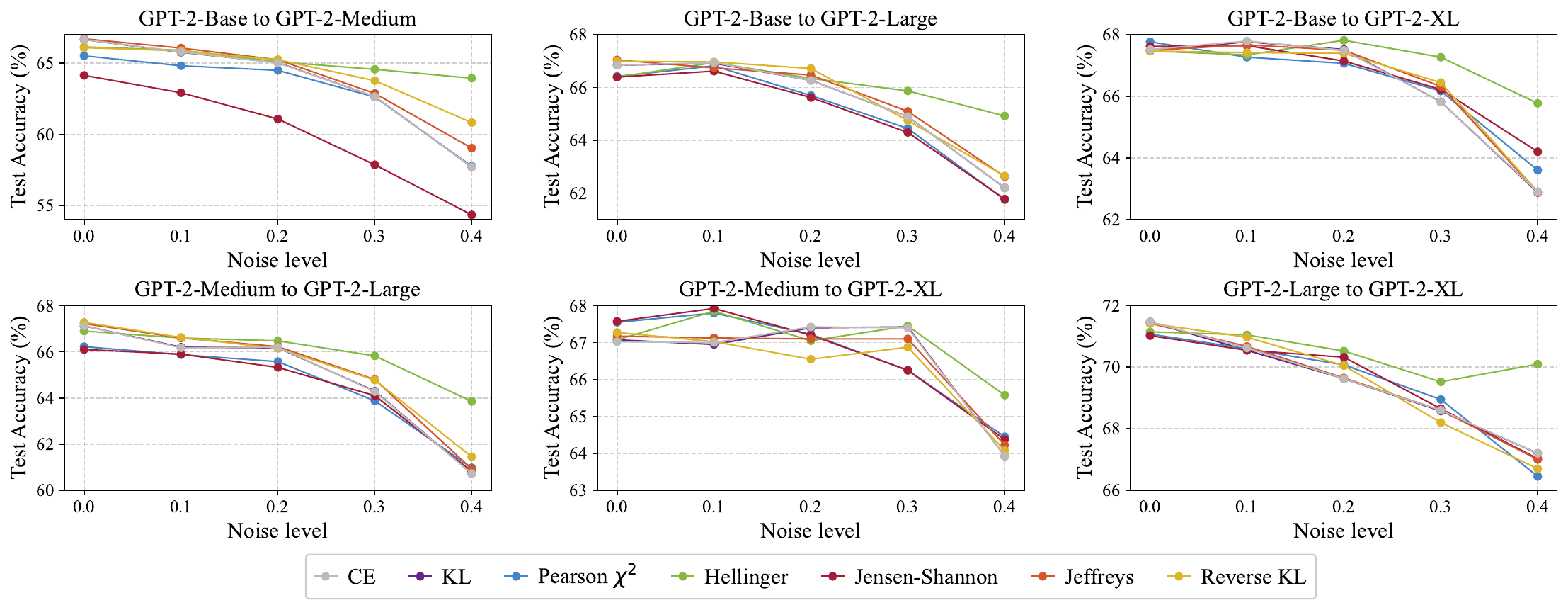}
        \caption{Results of GPT-2-series on HH-RLHF}
        \label{fig:helpful-noise}
    \end{subfigure}
    \vspace{-30pt}
    \caption{Results of GPT-2-series across varying noise levels on CAI-Harmless and HH-RLHF. 
    ``GPT-2-Base to GPT-2-Medium'' represents GPT-2-Base supervising GPT-2-Medium. 
    Figure legend from left to right: CE loss, KL divergence, Pearson $\chi^{2}$ divergence, squared Hellinger distance, JS divergence, Jeffreys divergence and reverse KL divergence.
    The results of KL divergence is not shown because it overlaps with CE. 
    }
    \label{fig:noise-comparison}
    \vspace{-12pt}
\end{figure*}

%% file: sections/experiments.tex
\section{Experiments}

\input{figures/fig_2}

\begin{table*}
\centering
\caption{Results of seven losses under 50\% label noise. 
``Base $\rightarrow$ Medium'' represents GPT-2-Base supervising GPT-2-Medium.
The best and runner-up results are highlighted in \textbf{bold} and \underline{underlined} text, respectively.}
\vspace{-2pt}
\label{tab:divergence_comparison}
\begin{tabular}{llccccccc}
\toprule
\textbf{Dataset} & \textbf{Models} & \textbf{$\chi^2$ } & \textbf{Hellinger} & \textbf{Jeffreys} & \textbf{JS} & \textbf{RKL} & \textbf{KL} & \textbf{CE} \\
\midrule
\multirow{6}{*}{CAI-Harmless} 
& Base $\rightarrow$ Medium &  \underline{44.18} & 37.50 & 34.88 & \textbf{44.93} & 30.18 & 40.38 & 40.38 \\
& Base $\rightarrow$ Large & \textbf{34.73} & 30.70 & 31.38 &  \underline{34.63} & 27.60 & 31.95 & 31.95 \\
& Base $\rightarrow$ XL &  \underline{33.20} & 31.05 & 30.93 & \textbf{33.70} & 27.15 & 32.30 & 32.58 \\
& Medium $\rightarrow$ Large &  \underline{34.58} & 28.48 & 29.88 & \textbf{34.60} & 30.00 & 32.08 & 32.08 \\
& Medium $\rightarrow$ XL & 34.40 & 26.35 & 30.28 & \textbf{33.55} & 29.80 &  \underline{31.53} &  \underline{31.53} \\
& Large $\rightarrow$ XL &  \underline{31.18} & 27.90 & 27.25 & 31.10 & \textbf{32.10} & 30.73 & 30.78 \\
\cmidrule(lr){1-9}
\multirow{6}{*}{HH-RLHF} 
& Base $\rightarrow$ Medium & 49.48 & 47.60 & 49.85 & 49.63 & \textbf{50.08} &  \underline{50.00} &  \underline{50.00} \\
& Base $\rightarrow$ Large &  \underline{51.45} & \textbf{55.73} & 50.70 & 52.00 & 50.63 & 50.30 & 50.33 \\
& Base $\rightarrow$ XL &  \underline{54.30} & \textbf{59.68} & 51.98 & 53.85 & 52.05 & 51.08 & 51.10 \\
& Medium $\rightarrow$ Large &  \underline{54.65} & 54.08 & 52.25 & \textbf{55.05} & 52.93 & 51.73 & 51.73 \\
& Medium $\rightarrow$ XL & 55.45 & \textbf{58.45} & 57.03 & 55.50 &  \underline{57.08} & 56.15 & 56.18 \\
& Large $\rightarrow$ XL & 51.65 & \textbf{51.73} & 50.15 &  \underline{51.68} & 49.03 & 51.23 & 51.18 \\
\bottomrule
\end{tabular}
\vspace{-10pt}
\end{table*}

To study different $f$-divergences, we employ six $f$-divergences (KL, Reverse KL, Jensen-Shannon, Jeffreys, Pearson $\chi^{2}$ divergence, and Squared Hellinger distance) along with cross-entropy as loss functions in W2SG.
We avoid using TV distance as an optimization objective due to its reliance on absolute values, which can lead to non-smooth optimization landscapes.
Our experimental configuration mainly follows the setup in prior works~\citep{yang2024super,yao2025revisiting}.

\subsection{Experimental Setup}

\paragraph{Datasets.}
Our experiments primarily focus on reward modeling tasks. 
We follow~\citet{yang2024super,yao2025revisiting} and use two binary classification datasets: CAI-Harmless~\citep{bai2022constitutional} and HH-RLHF~\citep{bai2022training}.
For each dataset, we construct paired examples by directly aligning the chosen and rejected responses while preserving their original preference labels.
To facilitate different stages of training and evaluation, each dataset is partitioned into three subsets of equal size (4K examples): (1) ground truth set for weak teacher model training, (2) weak supervision set for strong student model training, and (3) test set for final performance evaluation. All subsets are class-balanced to ensure fair comparisons.

\paragraph{Models.} 
We conduct experiments on GPT-2-series~\citep{radford2019language}, including GPT-2-Base, GPT-2-Medium, GPT-2-Large, and GPT-2-XL. 
We append a linear projection layer to the output of the pretrained model, followed by a Sigmoid activation that maps the representation to a probability in $[0, 1]$, indicating whether the response satisfies the harmlessness or helpfulness criteria. Notably, we fine-tune all model parameters without freezing any layers, enabling the model to fully adapt to the reward modeling tasks. 
To mitigate overfitting, we train for only a single epoch~\citep{burns2023weak}. 
The learning rate is $1 \times 10^{-5}$, with a batch size of 16 and a maximum input sequence length of 512. 

\paragraph{Label noise.}
To investigate the robustness of different loss functions to label noise, we introduce controlled noise by inverting the soft labels of a portion of the training samples. 
Specifically, for each selected sample, the original soft label $y$ is replaced with its complement $1-y$. 
We compare model performance under five noise levels (0.1, 0.2, 0.3, 0.4, and 0.5) with the clean data baseline. 
Notably, when the noise level is particularly high (about 0.5), the pseudo-labels provided by the weak model adversely affect the optimization, potentially leading to a performance even worse than random-guessing.

\paragraph{Confidence-enhancing regularization.} 
Following previous wisdom~\citep{burns2023weak}, we also investigate the regularized loss:
\begin{multline*} \label{aux_loss}
\mathcal{L}_{\text{AUX}} = \beta \, \text{CE}(G^w(x), G^{sw}(x)) + \\ (1 - \beta) \, \text{CE}(G^{sw}(x), \hat{G}^t(x)),
\end{multline*}
where $\hat{G}^t(x)$ denotes the hardened prediction of the student model. 
This auxiliary term is introduced to encourage stronger confidence in the student’s predictions and to prevent it from fully mimicking the weak teacher.
We further generalize this framework by replacing CE with six different $f$-divergences to comprehensively evaluate their effects under confidence-enhancing regularization.
All experiments employ a fixed regularization coefficient $\beta=0.5$ following~\citep{burns2023weak}, with $\beta$ warm-up during the first 50\% of training iterations.

\subsection{Main Results}

\paragraph{Clean data.} 
In the absence of introduced label noise (noise level is 0.0 in~\cref{fig:cai-noise} and~\cref{fig:helpful-noise}), $f$-divergence losses generally demonstrate comparable performance to CE loss. 
Specifically, both Jeffreys divergence (the sum of forward and reverse KL divergences) and reverse KL divergence, which incorporate the reverse KL term, are more likely to outperform CE loss. 
This finding aligns with prior work highlighting the advantages of reverse KL divergence in W2SG~\citep{yao2025revisiting}. 
However, JS divergence shows marginally worse performance, suggesting it may be less suitable for this setting.

\paragraph{Moderate noise.}
As noise level increases (noise level: 0.1, 0.2, 0.3, 0.4), CE loss performance degrades significantly. 
In contrast, Pearson $\chi^2$ divergence, Hellinger distance, Jeffreys divergence, and reverse KL outperform CE loss in many cases, with some maintaining nearly stable performance despite increasing noise. 
Notably, Hellinger distance emerges as particularly robust, achieving top performance across all noise levels and demonstrating superiority over both CE loss~\citep{burns2023weak} and reverse KL divergence~\cite{yao2025revisiting} in previous work. 
Hellinger distance sometimes even outperforms the runner-up loss by as much as 3\%.
The superior performance of Hellinger distance may stem from its relatively weaker mode-seeking behavior compared to other $f$-divergences~\citep{li2023mode}, which reduces its sensitivity to incorrect predictions.
Hellinger distance's bounded nature also makes it inherently less sensitive to outliers compared to unbounded losses like CE.

\paragraph{Extreme noise.}
As demonstrated in~\cref{tab:divergence_comparison}, under extreme label noise conditions (noise level is 0.5) where all loss functions degrade to near or even below random-guessing accuracy, we observe several noteworthy phenomena:
(1) Both Pearson $\chi^2$-divergence and JS divergence consistently outperform CE by a margin of at least 1\%, demonstrating superior robustness to severe label corruption.
(2) The Hellinger distance sometimes even achieve an improvement of up to 8\% over CE in certain cases. 
This significant advantage suggests that the Hellinger distance possesses unique properties that make it exceptionally resilient to some extreme noise conditions.
These empirical results suggest the advantages of $f$-divergences over CE for robust learning in high-noise regimes.

\subsection{Ablation Study}
As demonstrated in~\cref{fig:logconf-comparison}, reverse KL divergence and Jeffreys divergence consistently achieve slightly higher accuracy than CE loss in most scenarios. 
This observation aligns with the findings of~\citet{yao2025revisiting}, who demonstrated the particular advantages of reverse KL divergence  for confidence-enhancing regularization tasks.
While Hellinger distance shows superior performance in label noise settings, its benefits become less pronounced in this scenario. This is because the additional confidence-enhancing regularization effectively mitigates potential label noise, thereby limiting the relative improvement offered by Hellinger distance.

%% file: figures/fig_2.tex
\begin{figure*}[t]
    \centering
    \begin{subfigure}[b]{\linewidth}  
        \centering
        \includegraphics[width=\linewidth]{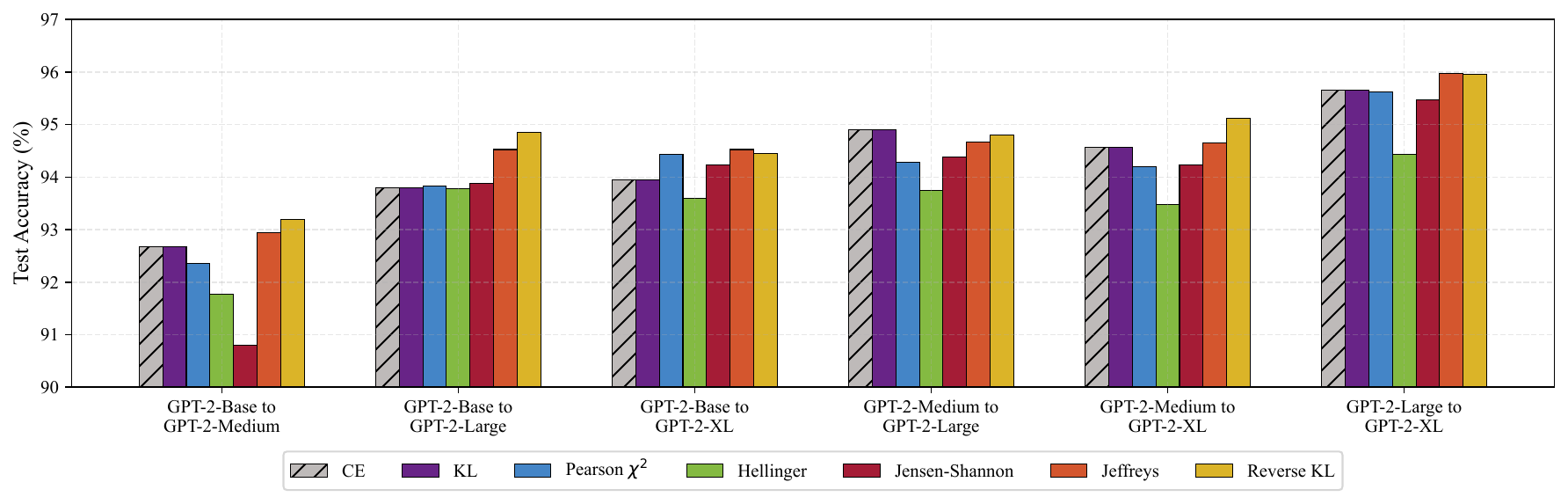}
        \caption{Results of GPT-2-series on CAI-Harmless}
        \label{fig:cai-logconf}
    \end{subfigure}
    
    \bigskip  
    
    \begin{subfigure}[b]{\linewidth}
        \centering
        \includegraphics[width=\linewidth]{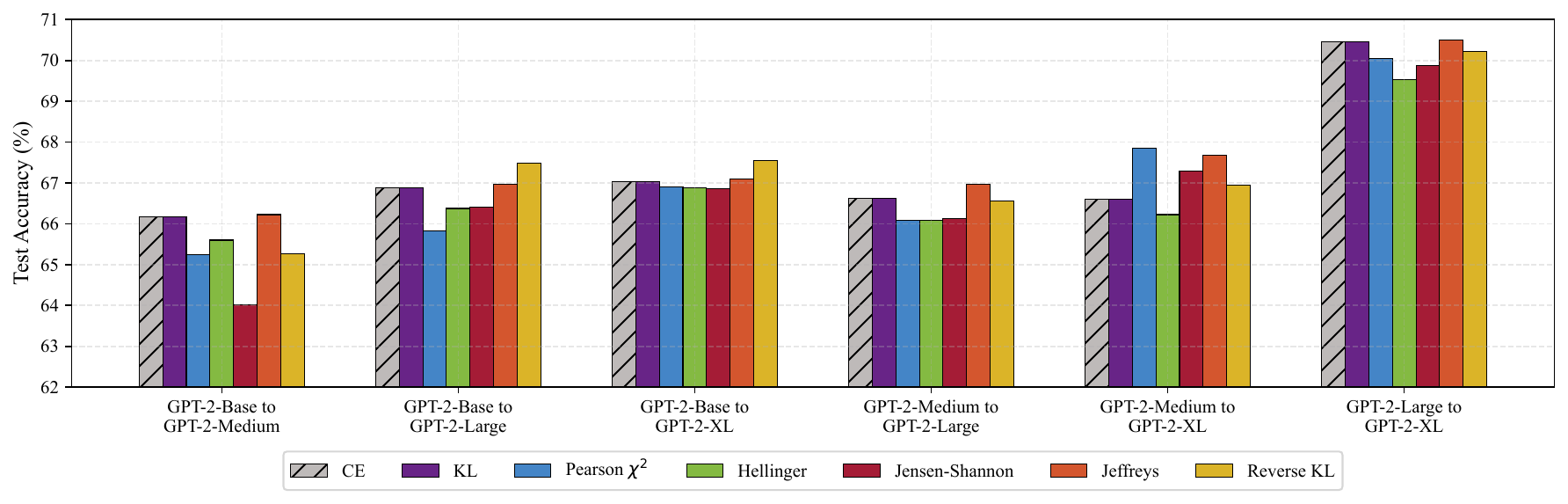}
        \caption{Results of GPT-2-series on HH-RLHF}
        \label{fig:helpful-logconf}
    \end{subfigure}
    \caption{Performance comparison of six $f$-divergence variants with confidence-enhancing regularization on CAI-Harmless and HH-RLHF, using CE as baseline.}
    \label{fig:logconf-comparison}
    \vspace{-10pt}
\end{figure*}

%% file: sections/conclusion.tex
\section{Conclusion}

In this paper, we present a novel information-theoretic framework for W2SG using $f$-divergence, demonstrating its effectiveness in improving model performance and noise tolerance while maintaining computational efficiency. 
Our theoretical analysis reveals fundamental limitations in W2SG, showing that the strong model's performance is intrinsically bounded by the weak supervisor's quality and the disagreement between models. 
We also establish the equivalence of different $f$-divergence losses under confidence-enhancing regularization.
Empirical results highlight the superiority of $f$-divergence losses in handling noisy supervision scenarios.
This work provides both theoretical insights and practical tools for advancing W2SG, offering a principled approach to leveraging weak supervision for training stronger models.

\newpage

\section*{Limitations}

While our work provides a principled framework for W2SG using $f$-divergence, several limitations warrant discussion.
First, while the theoretical advances and insights are valuable,~\cref{theorem:upper_lower_fdiv_sample} relies on asymptotic sample complexity bounds, which may not be sufficiently tight for extremely large language models. It is worth noting, however, that this limitation is not unique to our analysis but rather reflects a broader challenge shared by most theoretical frameworks in deep learning, particularly those applied to large language models.
Second, although we demonstrate the versatility of $f$-divergence through theoretical unification, our empirical validation is limited to classification tasks. The effectiveness of this approach for generative tasks remains an open question.

%% file: sections/appendix.tex
\section{Some Useful Definitions and Lemmas}

\subsection{Definitions}

\begin{definition}[Total variation distance] \label{def:tv_distance}
Given two probability distributions $P$ and $Q$, the Total Variation (TV) distance between $P$ and $Q$ is
$$\tv(P \| Q)= \frac{1}{2} \int \left| P(x)-Q(x) \right| d x.$$
\end{definition}

\begin{definition}[Mutual information]
For continuous random variables $X$ and $Y$ with joint probability density function $P(x,y)$ and marginal probability density functions $P(x)$ and $P(y)$, the mutual information is defined as:
$$I(X ; Y)= \iint P(x, y) \log \frac{P(x, y)}{P(x) P(y)} d x d y.$$
\end{definition}

\begin{definition}[Rademacher Complexity~\citep{bartlett2002rademacher}]
Let $P$ be a probability distribution over a domain space $Z$. The Rademacher complexity of the function class $\mathcal{F}$ w.r.t. $P$ for i.i.d. sample $S=\left(z_1, z_2, \cdots, z_m\right)$ with size $n$ is:
$$\mathcal{R}_n(\mathcal{F})=\mathbb{E}_{S \sim P^n}\left[\frac{1}{n} \mathbb{E}_{\boldsymbol{\sigma}}\left[\sup _{f \in \mathcal{F}} \sum_{i=1}^n \sigma_i f\left(z_i\right)\right]\right],$$
where the inner expectation is taken over $\sigma=\left\{\sigma_1, \sigma_2, \cdots, \sigma_m\right\}$ and they are independent random variables following the Rademarcher distribution, i.e., $P\left(\sigma_i=1\right)=P\left(\sigma_i=-1\right)=1 / 2$. 
The empirical Rademacher complexity is defined as,
$$\hat{\mathcal{R}}_n(\mathcal{F}, S)=\frac{1}{n} \mathbb{E}_{\boldsymbol{\sigma}}\left[\sup _{f \in \mathcal{F}} \sum_{i=1}^n \sigma_i f\left(z_i\right)\right].$$
\end{definition}

\begin{definition}[Subgaussian random variable]
A random variable $X \in \R$ is $\sigma$-subgaussian if for any $\lambda$, there holds 
$$\mathbb{E} \left[ e^{\lambda(X-\mathbb{E} [X])} \right] \leq \exp \left( \frac{\lambda^2 \sigma^2}{2} \right).$$
\end{definition}

\subsection{Lemmas}

\begin{lemma}[Corollary 2 from~\citet{gilardoni2010pinsker}] \label{lemma:tv_fdiv}
Let $P$ and $Q$ be two probability distributions, then there holds $\tv(P,Q) = \cO \left( \sqrt{\fdiv(P \| Q)} \right)$.
\end{lemma}

\begin{lemma}[Hoeffding's lemma] \label{hoeffding_lemma}
Let $X \in \R$ such that $a \leq X \leq b$. Then, for all $\lambda \in \mathbb{R}$,
$$\mathbb{E}\left[e^{\lambda(X-\mathbb{E}[X])}\right] \leq \exp \left(\frac{\lambda^2(b-a)^2}{8}\right).$$
\end{lemma}

\begin{lemma}[Lemma A.2 from~\citet{wang2022information}] \label{lemma:ziqiao}
Let $Q$ and $P$ be two probability distributions. If $g(\theta)$ is $R$-subgaussian, then
$$\left|\mathbb{E}_{\theta^{\prime} \sim Q}\left[g\left(\theta^{\prime}\right)\right]-\mathbb{E}_{\theta \sim P}[g(\theta)]\right| \leq \sqrt{2 R^2 \mathrm{D}_{\mathrm{KL}}(Q \| P)}.$$
\end{lemma}

\begin{lemma}[Talagrand’s Contraction Lemma~\citep{ledoux2013probability}] \label{talagrand}

Let $\Phi$ be l-Lipschitz functions from $\mathbb{R}$ to $\mathbb{R}, \sigma_1, \cdots, \sigma_m$ be Rademacher random variables, and $S=\left\{z_1, \cdots, z_m\right\}$ be a random i.i.d. sample. Then, for any hypothesis set $\mathcal{F}$ of real-valued functions, there holds
$$\hat{\mathcal{R}}_n(\Phi \circ \mathcal{F}, S) = \frac{1}{n} \mathbb{E}_{\boldsymbol{\sigma}}\left[\sup _{f \in \mathcal{F}} \sum_{i=1}^n \sigma_i\left(\Phi_i \circ f\right)\left(z_i\right)\right] \leq \frac{l}{n} \mathbb{E}_{\boldsymbol{\sigma}}\left[\sup _{f \in \mathcal{F}} \sum_{i=1}^n \sigma_i f\left(z_i\right)\right]=l \hat{\mathcal{R}}_n(\mathcal{F}, S).$$
\end{lemma}

\begin{remark}
However, since empirical Rademacher complexity $\hat{\mathcal{R}}_n$ is a sample-based estimate of the true Rademacher complexity $\mathcal{R}_n$, Talagrand’s Lemma applies at both the empirical and population levels. The key theoretical statement is about empirical Rademacher complexity in general, and it naturally extends to Rademacher complexity $\mathcal{R}_n$ in finite-sample settings.
\end{remark}

\section{Main Proof}

\subsection{Proof of~\cref{theorem:upper_lower_fdiv}} \label{proof:upper_lower_fdiv}

\begin{proof}
Recall that $\forall y = (y_1, \cdots, y_k)^T \in \cY$, there holds $0 < y_i < 1$.
$\forall g,h \in \cF$, we have that $\frac{g(x)}{h(x)}$ is bounded.
Furthermore, $\fdiv$ is also bounded.
And there holds

\begin{align*}
    \left| \dist(G^{sw}, G^\star) - \dist(G^w, G^\star) \right| & = \left| \bE_{x \sim \cP} \left[ \fdiv(G^{sw}(x) \| G^\star(x)) - \fdiv(G^w(x) \| G^\star(x)) \right] \right|
    \\ & = \left| \bE_{x \sim \cP} \left[ \sum_{i=1}^k [G^\star(x)]_i \cdot \left[ f \left( \frac{[G^{sw}(x)]_i}{[G^\star(x)]_i} \right) - f \left( \frac{[G^w(x)]_i}{[G^\star(x)]_i} \right) \right] \right] \right|
    \\ & = \left| \bE_{x \sim \cP} \left[ \sum_{i=1}^k [G^\star(x)]_i \cdot f'(\xi_i) \cdot \left| \frac{[G^{sw}(x)]_i}{[G^\star(x)]_i} - \frac{[G^w(x)]_i}{[G^\star(x)]_i} \right| \right] \right| \tag{$\xi_i$ is between $\frac{[G^{sw}(x)]_i}{[G^\star(x)]_i}$ and $\frac{[G^w(x)]_i}{[G^\star(x)]_i}$}
    \\ & \le \sup_{\xi_i}{|f'(\xi_i)|} \cdot \bE_{x \sim \cP} \left[ \sum_{i=1}^k [G^\star(x)]_i \cdot \left| \frac{[G^{sw}(x)]_i}{[G^\star(x)]_i} - \frac{[G^w(x)]_i}{[G^\star(x)]_i} \right| \right] 
    \\ & = \sup_{\xi_i}{|f'(\xi_i)|} \cdot \bE_{x \sim \cP} \left[ \sum_{i=1}^k \left| [G^{sw}(x)]_i - [G^w(x)]_i \right| \right]
    \\ & = 2\sup_{\xi_i}{|f'(\xi_i)|} \cdot \bE_{x \sim \cP} \left[ \tv \left(G^{sw}(x), G^w(x) \right) \right]
    \\ & = 2\sup_{\xi_i}{|f'(\xi_i)|} \cdot \bE_{x \sim \cP} \left[ \cO \left( \sqrt{\fdiv \left(G^{sw}(x) \| G^w(x) \right)} \right) \right] \tag{\cref{lemma:tv_fdiv}}
    \\ & = 2\sup_{\xi_i}{|f'(\xi_i)|} \cdot \cO \left[ \bE_{x \sim \cP} \left( \sqrt{\fdiv \left(G^{sw}(x) \| G^w(x) \right)} \right) \right] \tag{The linear nature of expectation}
    \\ & \le 2\sup_{\xi_i}{|f'(\xi_i)|} \cdot \cO \left( \sqrt{ \bE_{x \sim \cP} \; \fdiv \left(G^{sw}(x) \| G^w(x) \right)} \right) \tag{Jensen's inequality}
    \\ & \le 2\sup_{\xi_i}{|f'(\xi_i)|} \cdot \cO \left( \sqrt{ \dist \left(G^{sw}, G^w \right)} \right) 
    \\ & = \cO \left( \sqrt{ \dist \left(G^{sw}, G^w \right)} \right). \tag{$|f'(\xi_i)|$ is bounded}
\end{align*}

\end{proof}

\subsection{Proof of~\cref{theorem:upper_lower_fdiv_sample}} \label{proof:upper_lower_fdiv_sample}

We first prove the uniform convergence result (\cref{lem:uniform-convergence}).
It is a natural extension of previous work~\citep{charikar2024quantifying, yao2025understanding, yao2025revisiting} on other loss functions that also satisfy boundedness and the Lipschitz condition. For completeness, we sketch the main proof here.

\begin{lemma}[Uniform convergence]
\label{lem:uniform-convergence}
Let $(x_1,y_1),\dots,(x_n, y_n)$ be an i.i.d. training sample, where each $x_i \sim \cP$ and $y_i = G^w(x_i)$ for a target function $G^w$. For a fixed strong model representation $h^s$, we employ $f$-divergence loss in W2SG:
\begin{align*}
    & g^{sw} = \argmin_{g \in \cF_{s}}\; \dist(g \circ h^s, G^w),
    \\ & \hat{g}^{sw} = \argmin_{g \in \cF_s} \disthat(g \circ h^s, G^w).
\end{align*}
Assume that the range of $G^w$ and functions in $\cF_s$ is absolutely bounded. Then, with probability at least $1-\delta$ over the draw of $(x_1,y_1),\dots,(x_n, y_n)$, we have
\begin{align*}
    \left|\dist(\hat{G}^{sw}, G^w) - \dist(G^{sw}, G^w) \right| \le \cO\left(\sqrt{\frac{\cC_{\cF_s}}{n}}\right) + \cO\left(\sqrt{\frac{\log(1/\delta)}{n}}\right),
\end{align*}
where $\cC_{\cF_s}$ is a constant capturing the complexity of the function class $\cF_s$, and the asymptotic notation is with respect to $n \to \infty$.
\end{lemma}

\begin{proof}[Proof of~\cref{lem:uniform-convergence}]
The proof is strongly motivated by lemma 4 in~\citet{charikar2024quantifying}.
Recall again that $\forall y = (y_1, \cdots, y_k)^T \in \cY$, there holds $0 < y_i < 1$.
Therefore, $\fdiv$ is both bounded and Lipschitz in both arguments.
Note that
\begin{multline}
    \label{eqn:risk-decomposition}
    \dist(\hat{G}^{sw}, G^w) - \dist(G^{sw}, G^w) = \underbrace{\dist(\hat{G}^{sw}, G^w) - \disthat(\hat{G}^{sw}, G^w)}_{a} + \\ \underbrace{\disthat(\hat{G}^{sw}, G^w) - \disthat(G^{sw}, G^w)}_{b} + \underbrace{\disthat(G^{sw}, G^w) - \dist(G^{sw}, G^w)}_{c}.
\end{multline}
By the definition of $\hat{G}^{sw}$, the second term $b\le 0$ in \eqref{eqn:risk-decomposition}. 
Therefore,
\begin{align} \label{proof:ineq:conv}
    \left| \dist(\hat{G}^{sw}, G^w) - \dist(G^{sw}, G^w) \right| \le |a| +|c|.
\end{align}
The terms $a$ and $c$ measure the difference between the empirical risk and true population risk, and can be controlled by a standard uniform convergence argument.

Given the dataset $\{(x_i,y_i)\}_{i=1}^n$, where $x_i \sim \cP$ and $y_i = G^w(x_i)$. 
According to statistical learning theory literature~\citep{bartlett2002rademacher}, it first holds that with probability at least $1-\delta$,
\begin{align*}
    \sup_{g \in \cF_s} |\disthat(g \circ h^s, G^w) - \dist(g \circ h^s, G^w)| &\le O \left( \mathcal{R}_n (\ell(\cF_s))\right) + \cO\left(\sqrt{\frac{\log(1/\delta)}{n}}\right),
\end{align*}
where $\mathcal{R}_n (\ell(\cF_s))$ is the \textit{Rademacher complexity} of the loss class of $\cF_s$.
\begin{align*}
    \mathcal{R}_n (\ell(\cF_s)) &= \E_{S} \E_{\eps_i \sim \{-1,1\}} \sup_{g \in \cF_s} \frac{1}{n}\sum_{i=1}^n \eps_i \cdot \fdiv(g \circ h^s(x_i), y_i).
\end{align*}
Notice again that the model output space $\cY = \{ y \in \R| 0 < y < 1 \}$.
We can then use the assumption that the range of $G^w$ and functions in $\cF_s$ are absolutely bounded, which implies that $\ell$ is both bounded and Lipschitz in both arguments. 
This allows us to use the contraction principle (\cref{talagrand}) so as to move from the Rademacher complexity of the loss class $\ell(\cF_s)$ to that of $\cF_s$ itself, and claim that with probability at least $1-\delta$,
\begin{align}
    \sup_{g \in \cF_s} |\disthat(g \circ h^s, G^w) - \dist(g \circ h^s, G^w)| & \le O \left( \mathcal{R}_n (\cF_s)\right) + \cO\left(\sqrt{\frac{\log(1/\delta)}{n}}\right) \label{eqn:rademacher-cxty-bound}
\end{align}
Finally, the Rademacher complexity $\mathcal{R}_n (\cF_s)$ can be upper bounded by a quantity known as the \textit{worst-case Gaussian complexity} of $\cF_s$; in any case, for a majority of parametric function classes $\cF_s$, this quantity scales as $\sqrt{\frac{\cC_{\cF_s}}{n}}$~\citep{bartlett2002rademacher}, where $\cC_{\cF_s}$ is a constant capturing the inherent complexity of $\cF_s$. 
Plugging this into \eqref{eqn:rademacher-cxty-bound} and considering $f=\hat{G}^{sw}$ or $f=G^{sw}$ in this inequality, we have
\begin{align}
    \underbrace{\left| \disthat(\hat{G}^{sw}, G^w) - \dist(\hat{G}^{sw}, G^w) \right|}_{|a|} \le \cO\left(\sqrt{\frac{\cC_{\cF_s}}{n}}\right) + \cO\left(\sqrt{\frac{\log(1/\delta)}{n}}\right), \label{proof:lemma-uniform-eq1} \\
    \underbrace{\left| \disthat(G^{sw}, G^w) - \dist(G^{sw}, G^w) \right|}_{|c|} \le \cO\left(\sqrt{\frac{\cC_{\cF_s}}{n}}\right) + \cO\left(\sqrt{\frac{\log(1/\delta)}{n}}\right).
\end{align}

Finally, substitute it into~\eqref{proof:ineq:conv} and we can obtain the desired bound.
\end{proof}

\begin{proof}[Proof of~\cref{theorem:upper_lower_fdiv_sample}]

Substitute $G^{sw}$ with $\hat{G}^{sw}$ in the proof of~\cref{theorem:upper_lower_fdiv} and we can prove that
\begin{align}
    \left| \dist(\hat{G}^{sw}, G^\star) - \dist(G^w, G^\star) \right| = \cO \left( \sqrt{\dist(\hat{G}^{sw}, G^w)} \right).
\end{align}

Combine~\eqref{proof:lemma-uniform-eq1} with the above and we prove the final result:
\begin{align*}
    & \left| \dist(\hat{G}^{sw}, G^\star) - \dist(G^w, G^\star) \right| 
    \\ = & \cO \left( \sqrt{\dist(\hat{G}^{sw}, G^w)} \right)
    \\ \le & \cO \left( \sqrt{\disthat(\hat{G}^{sw}, G^w) + \cO\left(\sqrt{\frac{\cC_{\cF_s}}{n}}\right) + \cO\left(\sqrt{\frac{\log(1/\delta)}{n}}\right)} \right)
    \\ \le & \cO \left( \sqrt{\disthat(\hat{G}^{sw}, G^w)} + \sqrt{\cO\left(\sqrt{\frac{\cC_{\cF_s}}{n}}\right)} + \sqrt{ \cO\left(\sqrt{\frac{\log(1/\delta)}{n}}\right)} \right)
    \\ \le & \cO \left( \sqrt{\disthat(\hat{G}^{sw}, G^w)} + \left(\frac{\cC_{\cF_s}}{n}\right)^{\frac{1}{4}} + \left(\frac{\log(1/\delta)}{n}\right)^{\frac{1}{4}} \right)
    \\ = & \cO \left( \sqrt{\disthat(\hat{G}^{sw}, G^w)} \right) + \cO \left(\frac{\cC_{\cF_s}}{n}\right)^{\frac{1}{4}} + \cO \left(\frac{\log(1/\delta)}{n}\right)^{\frac{1}{4}}.
\end{align*}

\end{proof}

\subsection{Proof of~\cref{corollary:information}} \label{proof:inf_theoretic}

\begin{lemma} \label{lemma:upper_lower_fdiv}
Given the data domain $\cX$, output domain $\cY$, models $G^w, G^\star$ and disagreement $\dist(\cdot, \cdot)$ defined above. 
For any strong model $\hat{G}^{sw}$, there holds
\begin{align}
    \left| \mathbb{E}_{\theta,Z} \left[\dist(\hat{G}^{sw}, G^\star) - \dist(G^w, G^\star) \right] \right| = \cO \left(\sqrt{\mathbb{E}_{\theta,Z} \left[ \dist(\hat{G}^{sw}, G^w) \right]} \right).
\end{align}
\end{lemma}

\begin{proof}[Proof of~\cref{lemma:upper_lower_fdiv}]
According to the linear nature of expectation, take expectations over $\theta,Z$ at both sides of the proof in~\cref{proof:upper_lower_fdiv} and we can prove that
\begin{align*}
    \left| \mathbb{E}_{\theta,Z} \left[\dist(\hat{G}^{sw}, G^\star) - \dist(G^w, G^\star) \right] \right| & = \cO \left( \mathbb{E}_{\theta,Z} \sqrt{\dist(\hat{G}^{sw}, G^w)} \right) \\
    & \le \cO \left(\sqrt{\mathbb{E}_{\theta,Z} \left[ \dist(\hat{G}^{sw}, G^w) \right]} \right) \tag{Jensen's inequality}
\end{align*}
\end{proof}

\begin{theorem} \label{corollary:upper_lower_fdiv}
Given the data domain $\cX$, output domain $\cY$, models $G^w$ and disagreement $\dist(\cdot, \cdot)$ defined above. 
Let $\theta \sim \cP_\theta$ be the parameters of the strong model $\hat{G}^{sw}$, and $Z=(X_1, \cdots, X_n) \sim \cP_Z$, where $\cP_Z = \bigotimes_{i=1}^n \cP_X$.
For any strong model $\hat{G}^{sw}$, there holds
\begin{align}
    \left| \mathbb{E}_{\theta,Z} \left[ \dist(\hat{G}^{sw},G^w) - \disthat(\hat{G}^{sw},G^w) \right] \right| = \cO \left( \sqrt{\frac{I(\theta;Z)}{n}} \right).
\end{align}
\end{theorem}

\begin{proof}[Proof of~\cref{corollary:upper_lower_fdiv}]
Substitute the following equations into~\cref{lemma:ziqiao}.
\begin{align*}
    & g(\theta) = \frac{1}{n} \sum_{j=1}^n \ell(\theta, x_j) = \frac{1}{n} \sum_{j=1}^n \fdiv(\hat{G}_\theta^{sw}(x_j) \| G^w(x_j)), \\
    & P = \cP_\theta \otimes \cP_Z, \\
    & Q = \cP_{\theta,Z}.
\end{align*}
We obtain
\begin{align} \label{ineq:infor-base}
    \left|\mathbb{E}_{\theta,Z}\left[\frac{1}{n}\sum_{j=1}^n \ell(\theta, x_j)\right]-\expect_\theta \expect_Z \left[\frac{1}{n}\sum_{j=1}^n \ell(\theta, x_j) \right]\right| & \leq \sqrt{2 R^2 \mathrm{D}_{\mathrm{KL}} \left(\cP_{\theta,Z} \| \cP_\theta \otimes \cP_Z \right)} \notag
    \\ & = \sqrt{2 R^2 I(\theta;Z)}.
\end{align}

Firstly, 
\begin{align*}
    \expect_\theta \expect_Z \left[\frac{1}{n}\sum_{j=1}^n \ell(\theta, x_j) \right] & = \expect_\theta \mathbb{E}_{X_1} \cdots \mathbb{E}_{X_n} \left[\frac{1}{n}\sum_{j=1}^n \ell(\theta, x_j) \right] \\
    & = \expect_\theta \left[\frac{1}{n}\sum_{j=1}^n \mathbb{E}_{X_j} \ell(\theta, x_j) \right] \\
    & = \expect_\theta \left[\frac{1}{n}\sum_{j=1}^n \mathbb{E}_{X_j} \fdiv(\hat{G}^{sw}(x_j) \| G^w(x_j)) \right] \\
    & = \expect_\theta \left[\mathbb{E}_{X_j} \fdiv(\hat{G}^{sw}(x_j) \| G^w(x_j)) \right] \\
    & = \expect_\theta \left[ \dist(\hat{G}^{sw},G^w) \right] \\
    & = \mathbb{E}_{\theta,Z} \left[ \dist(\hat{G}^{sw},G^w) \right]
\end{align*}

Secondly,
\begin{align*}
    \mathbb{E}_{\theta,Z}\left[\frac{1}{n}\sum_{j=1}^n \ell(\theta, x_j)\right] & = \mathbb{E}_{\theta,Z} \left[\frac{1}{n}\sum_{j=1}^n \fdiv(\hat{G}^{sw}(x_j) \| G^w(x_j)) \right] \\
    & = \mathbb{E}_{\theta,Z} \left[ \disthat(\hat{G}^{sw},G^w) \right]
\end{align*}

Substitute them back into~\eqref{ineq:infor-base} and we obtain
\begin{align} \label{ineq:infor_med}
    \left| \mathbb{E}_{\theta,Z} \left[ \dist(\hat{G}^{sw},G^w) - \disthat(\hat{G}^{sw},G^w) \right] \right| \le \sqrt{2 R^2 I(\theta;Z)}.
\end{align}

According to~\cref{hoeffding_lemma} (i.e., a bounded variable is subgaussian) and the fact that $\fdiv$ is bounded, we know that $\fdiv$ is subgaussian.
If $\fdiv$ is $\sigma$-subgaussian, then $g(\cdot)$ is $\frac{\sigma}{\sqrt{n}}$-subgaussian.
Let $R=\frac{\sigma}{\sqrt{n}}$ in~\eqref{ineq:infor_med} and we obtain:
\begin{align*}
    \left| \mathbb{E}_{\theta,Z} \left[ \dist(\hat{G}^{sw},G^w) - \disthat(\hat{G}^{sw},G^w) \right] \right| \le \cO \left( \sqrt{\frac{I(\theta;Z)}{n}} \right).
\end{align*}
\end{proof}

\begin{proof}[Proof of~\cref{corollary:information}]
Taking~\cref{lemma:upper_lower_fdiv} and~\cref{corollary:upper_lower_fdiv} together, we have
\begin{align*}
    \left| \mathbb{E}_{\theta,Z} \left[\dist(\hat{G}^{sw}, G^\star) - \dist(G^w, G^\star) \right] \right| & = \cO \left(\sqrt{\mathbb{E}_{\theta,Z} \left[ \dist(\hat{G}^{sw}, G^w) \right]} \right) \tag{\cref{lemma:upper_lower_fdiv}}
    \\ & \le \cO \left(\sqrt{\mathbb{E}_{\theta,Z} \left[\disthat(\hat{G}^{sw}, G^w) \right]+ \cO \left( \sqrt{\frac{I(\theta;Z)}{n}} \right)} \right) \tag{\cref{corollary:upper_lower_fdiv}}
    \\ & \le \cO \left(\sqrt{\mathbb{E}_{\theta,Z} \left[ \disthat(\hat{G}^{sw}, G^w) \right]} + \sqrt{\cO \left( \sqrt{\frac{I(\theta;Z)}{n}} \right)}\right)
    \\ & = \cO \left(\sqrt{\mathbb{E}_{\theta,Z} \left[ \disthat(\hat{G}^{sw}, G^w) \right]}\right) + \cO \left(\frac{I(\theta;Z)}{n} \right)^{\frac{1}{4}}.
\end{align*}
\end{proof}

\subsection{Proof of~\cref{theorem:equivalence}} \label{proof:equivalence}

\begin{proof}

The proof is motivated by~\citet{daunas2024equivalence}.
First, notice that the prediction of the model for every data point $x$ is constructed as a probability distribution.
For simplicity of notation, denote the prediction probability of the $i$-th data point $x_i$ of $G^{sw}_\theta$ and $G^w$ as $P_i$ and $Q_i$, respectively.
Also, rewrite $L(\theta, x_i)$ as $L_i(\theta)$ and denote $\frac{d P_i}{d Q_i}(\theta)$ as $g_i(\theta)$.
The minimization goal can be formulated as an information-theoretic approach similar to~\citet{daunas2024equivalence}.
\begin{align*}
& \min_{\theta} \hat{R}_{f_1} \left(G^{sw}_\theta, G^w\right) + \alpha \cdot \expect_\theta \left[ \frac{1}{n} \sum_{i=1}^n L(\theta, x_i) \right]
\\ \Leftrightarrow & \min_{\theta} \frac{1}{n} \sum_{j=1}^n \mathrm{D}_{f_1}(G^{sw}_\theta(x_j) \| G^w(x_j)) + \alpha \cdot \expect_\theta \left[ \frac{1}{n} \sum_{i=1}^n L(\theta, x_i) \right] 
\\ \Leftrightarrow & \min_{\theta} \frac{1}{n} \sum_{j=1}^n \int f_1 \left( \frac{d P_i}{d Q_i}(\theta)\right) d Q_i(\theta) +  \alpha \cdot \frac{1}{n} \sum_{i=1}^n \int L_i(\theta) \; d P_i(\theta)
\\ \Leftrightarrow & \min_{g_i(\theta), i \in [n]} \frac{1}{n} \sum_{j=1}^n \int f_1 \left( g_i(\theta)\right) d Q_i(\theta) +  \alpha \cdot \frac{1}{n} \sum_{i=1}^n \int L_i(\theta) g_i(\theta) \; d Q_i(\theta)
\\ \Leftrightarrow & \min_{g_i(\theta), i \in [n]} \frac{1}{n} \sum_{j=1}^n \int \left[ f_1 \left( g_i(\theta)\right) +  \alpha \cdot L_i(\theta) g_i(\theta) \right] \; d Q_i(\theta).
\end{align*}

We first find the solution of the constrained optimization problem
\begin{subequations}
\begin{align}
& \min_{g_i(\theta), i \in [n]} \frac{1}{n} \sum_{j=1}^n \int \left[ f_1 \left( g_i(\theta)\right) +  \alpha \cdot L_i(\theta) g_i(\theta) \right] \; d Q_i(\theta), \label{opt:goal}
\\ & \text{s.t.} \qquad \int g_i(\theta) \; dQ_i(\theta) = 1, \qquad \text{for any} \; i \in [n]. \label{opt:constraint}
\end{align}
\end{subequations}

Introducing Lagrange multipliers $\beta_1, \cdots, \beta_n$, we have the following
\begin{align} \label{eq:lagrange}
\mathcal{L}_1 \left( \mathbf{g}, \bm{\beta} \right) & = \mathcal{L}_1 \left( \left( g_1(\theta), \cdots, g_n(\theta) \right), \left( \beta_1, \cdots, \beta_n \right) \right) \notag
\\ & = \frac{1}{n} \sum_{j=1}^n \int \left[ f_1 \left( g_i(\theta)\right) +  \alpha \cdot L_i(\theta) g_i(\theta) + n \beta_i g_i(\theta) \right] \; d Q_i(\theta) - \sum_{i=1}^n \beta_i.
\end{align}

We consider the direction of $\mathbf{g}(\theta)$ as $\hat{\mathbf{g}}(\theta)=\left( \hat{g}_1(\theta), \cdots, \hat{g}_n(\theta) \right)$ and compute the Gateaux differential of $\mathcal{L}_1$ as
\begin{align}
    \left.\partial \mathcal{L}_1 \left(\mathbf{g}, \bm{\beta} ; \ \hat{\mathbf{g}} \right) \triangleq \frac{\mathrm{d}}{\mathrm{~d} \gamma} \mathcal{L}_1 \left( \mathbf{g}+\gamma \hat{\mathbf{g}}, \bm{\beta} \right) \right|_{\gamma=0}.
\end{align}

We have
\begin{multline*}
    \mathcal{L}_1 \left( \mathbf{g}+\gamma \hat{\mathbf{g}}, \bm{\beta} \right) = \frac{1}{n} \sum_{j=1}^n \int \Big[ f_1 \left( g_i(\theta)+ \gamma \hat{g}_i(\theta)\right) + \\ \left( \alpha L_i(\theta) + n \beta_i \right) \left( g_i(\theta)+ \gamma \hat{g}_i(\theta) \right) \Big] \; d Q_i(\theta) - \sum_{i=1}^n \beta_i,
\end{multline*}
and
\begin{align*}
\partial \mathcal{L}_1 \left(\mathbf{g}, \bm{\beta} ; \ \hat{\mathbf{g}} \right) = & \frac{\mathrm{d}}{\mathrm{~d} \gamma} \mathcal{L}_1 \left( \mathbf{g}+\gamma \hat{\mathbf{g}}, \bm{\beta} \right) \Bigg|_{\gamma=0}
\\ = & \frac{1}{n} \sum_{j=1}^n \int \Big[ f'_1 \left( g_i(\theta)+ \gamma \hat{g}_i(\theta)\right) \hat{g}_i(\theta) + \left( \alpha L_i(\theta) + n \beta_i \right) \hat{g}_i(\theta) \Big] \; d Q_i(\theta) \Bigg|_{\gamma=0}
\\ = & \frac{1}{n} \sum_{j=1}^n \int \Big[ \hat{g}_i(\theta) \cdot \left( \alpha L_i(\theta) + n \beta_i + f'_1 \left( g_i(\theta)\right) \right) \Big] \; d Q_i(\theta).
\end{align*}

Let $\partial \mathcal{L}_1 \left(\mathbf{g}, \bm{\beta} ; \ \hat{\mathbf{g}} \right)=0$ for any $\hat{\mathbf{g}}(\theta)$, the solution of~\cref{eq:lagrange} is the stationary point and satisfies
\begin{align*}
\alpha L_i(\theta) + n \beta_i + f'_1 \left( g_i(\theta)\right) = 0, \qquad \text{for any} \; i=1, \cdots, n.
\end{align*}

Also, note that~\eqref{opt:goal} and~\eqref{opt:constraint} are convex with respect to $g$. 
It means that the solution $\tilde{g}_i$ is unique:
\begin{align*}
\tilde{g}_i(\theta) = (f'_1)^{-1} \left(-\alpha L_i(\theta) -n \beta_i \right), \qquad \text{for any} \; i=1, \cdots, n,
\end{align*}
where $\beta_i$ is chosen to make $\int \tilde{g}_i(\theta) \; dQ_i(\theta) = 1$. 
To facilitate the explicit presentation of solution, we also introduce the normalization function as Section IV in~\citet{daunas2024equivalence}.
The intuition is that a legal Lagrange multiplier $\beta_i$ needs to satisfy the normalization property and it is a function of $\alpha$. So we can replace $\beta_i$ with a normalization function.
Specifically, for all $i=1, \cdots, n$, let $N_{Q,i}(\alpha): \R \to \R$ be the normalization functions of the optimization problem~\eqref{opt:goal} and~\eqref{opt:constraint} such that
\begin{align*}
    \int (f'_1)^{-1} \left(-\alpha L_i(\theta) -n N_{Q,i}(\alpha) \right) \; dQ_i(\theta) = 1, \qquad \text{for any} \; i=1, \cdots, n.
\end{align*}
Therefore, the solution of the optimization problem is
\begin{align} \label{eqn:opt_solution}
\tilde{g}_i(\theta) = (f'_1)^{-1} \left(-\alpha L_i(\theta) -n N_{Q,i}(\alpha) \right), \qquad \text{for any} \; i=1, \cdots, n.
\end{align}

Similarly, consider another minimization goal
\begin{align*}
& \min_{\theta} \hat{R}_{f_2} \left(G^{sw}_\theta, G^w\right) + \alpha \cdot \expect_\theta \left[ \frac{1}{n} \sum_{i=1}^n v_i\left( L(\theta, x_i) \right) \right]
\\ \Leftrightarrow & \min_{g_i(\theta), i \in [n]} \frac{1}{n} \sum_{j=1}^n \int \left[ f_2 \left( g_i(\theta)\right) +  \alpha \cdot v_i \left( L_i(\theta) \right) g_i(\theta) \right] \; d Q_i(\theta).
\end{align*}

Likewise, let $N'_{Q,i}(\alpha): \R \to \R$ be the normalization function of the above optimization problem such that
\begin{align*}
    \int (f'_2)^{-1} \left(-\alpha L_i(\theta) -n N'_{Q,i}(\alpha) \right) \; dQ_i(\theta) = 1, \qquad \text{for any} \; i=1, \cdots, n.
\end{align*}

Using the same proof technique and we can obtain the solution like~\cref{eqn:opt_solution}:
\begin{align} \label{eqn:opt_solution_f2}
\mathring{g}_i(\theta) = (f'_2)^{-1} \left(-\alpha \cdot v_i \left( L_i(\theta) \right) -n N'_{Q,i}(\alpha) \right), \qquad \text{for any} \; i=1, \cdots, n.
\end{align}

Define the transformation function
\begin{align} \label{transform:func}
    v_i(x)=-\frac{1}{\alpha} f'_2\left( (f'_1)^{-1} \left( -\alpha x - n N_{Q,i}(\alpha) \right) \right) - \frac{n}{\alpha} N'_{Q,i}(\alpha), \qquad \text{for any} \; i=1, \cdots, n.
\end{align}

Substitute $v(x)$ into $\mathring{g}_i(\theta)$ and we have
\begin{align*}
    \mathring{g}_i(\theta) & = (f'_2)^{-1} \left(-\alpha \cdot v_i \left( L_i(\theta) \right) -n N'_{Q,i}(\alpha) \right)
    \\ & = (f'_2)^{-1} \left(f'_2\left( (f'_1)^{-1} \left( -\alpha L_i(\theta) - n N_{Q,i}(\alpha) \right) \right) \right)
    \\ & = (f'_1)^{-1} \left( -\alpha L_i(\theta) - n N_{Q,i}(\alpha) \right)
    \\ & = \tilde{g}_i(\theta).
\end{align*}

Therefore, using the transformation function in~\cref{transform:func}, the solutions \cref{eqn:opt_solution} and~\cref{eqn:opt_solution_f2} are equivalent to each other, i.e., $\tilde{g}_i(\theta) = \mathring{g}_i(\theta)$.
The proof is complete.

\end{proof}